\documentclass{article}

\usepackage{arxiv}

\usepackage[utf8]{inputenc}
\usepackage[T1]{fontenc}
\usepackage{hyperref}
\usepackage{url}
\usepackage{booktabs}
\usepackage{amsmath,amssymb,amsfonts,mathtools,bm}
\usepackage{microtype}
\usepackage{graphicx}
\usepackage{natbib}
\usepackage{doi}
\usepackage{xcolor}

\DeclareMathOperator{\Var}{Var}

\usepackage{amsthm}
\newtheorem{theorem}{Theorem}
\newtheorem{proposition}[theorem]{Proposition}
\newtheorem{lemma}[theorem]{Lemma}
\newtheorem{corollary}[theorem]{Corollary}
\theoremstyle{remark}
\newtheorem{remark}{Remark}

\title{Expected Survival-Time Bounds for Robust Optimization Over Time under Isotropic Gaussian Dynamics}

\author{
	\href{https://orcid.org/0000-0003-3267-6753}{\includegraphics[scale=0.06]{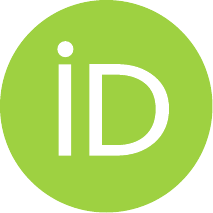}\hspace{1mm}Pavel Novoa-Hern{\'a}ndez} \\
	Dept.\ Computer Engineering and Systems\\
	Universidad de La Laguna\\
	Spain \\
	\texttt{pnovoahe@ull.edu.es}
}

\renewcommand{\shorttitle}{Survival-time bounds for ROOT}

\hypersetup{
pdftitle={Expected Survival-Time Bounds for Robust Optimization Over Time under Isotropic Gaussian Dynamics},
pdfsubject={60G50, 60G40, 90C59, 68T20},
pdfauthor={Pavel Novoa-Hernandez},
pdfkeywords={robust optimization over time, survival time, first-exit time, Gaussian random walk, deployment persistence, evolutionary dynamic optimization},
}

\begin{document}
\maketitle

\begin{abstract}
Robust Optimization Over Time (ROOT) is a recent branch of
evolutionary dynamic optimization that seeks solutions capable of remaining
effective across multiple consecutive environments. Unlike the traditional
track-the-moving-optimum (TMO) paradigm, which reoptimizes after every
environmental change, ROOT explicitly values persistence and aims to reduce
the need for frequent redeployments. Although the field has grown
considerably, most contributions remain algorithmic and empirical, leaving
several fundamental properties of ROOT problems poorly understood from a
theoretical perspective.
One such property is \emph{survival time}, defined as the number of future
environments in which a deployed solution continues to satisfy a prescribed
quality threshold. While survival time is frequently used as a measure of
temporal robustness, little is known about how its expected value depends on
environmental dynamics, deployment quality, or problem characteristics.
This paper addresses that gap for a fixed deployed solution under isotropic
Gaussian environmental dynamics. Modeling survival as a discrete first-exit
problem, we derive a rigorous lower bound and a computable multi-step upper
bound for the expected survival time. The analysis shows that expected
survival scales as $\Theta(\sigma^{-2})$ in slowly varying environments and
approaches its minimum value of one future change in high dimensions.
A comprehensive Monte Carlo study validates the theoretical predictions,
examines sensitivity to modeling assumptions and parameter uncertainty, and
illustrates how the proposed bounds can support deployment decisions after
optimization. Rather than introducing a new ROOT algorithm, the paper
provides an analytical characterization of deployment lifetime and a
decision-oriented framework for determining when a required deployment
horizon can be guaranteed, ruled out, or remains analytically unresolved.
\end{abstract}

\keywords{robust optimization over time \and survival time \and
first-exit time \and Gaussian random walk \and deployment persistence \and
evolutionary dynamic optimization}

\section{Introduction}\label{sec:introduction}

Evolutionary Dynamic Optimization (EDO) studies optimization problems whose objective function changes over time, typically represented as a sequence of environments, each inducing a different fitness landscape \citep{cruz2011optimization,nguyen2012evolutionary}. The dominant paradigm in this setting is \emph{tracking the moving optimum} (TMO), whereby an evolutionary algorithm repeatedly searches for the current optimum after each
environmental change and replaces the deployed solution with the newly found one. This strategy has been highly successful in benchmark studies and remains the prevailing perspective in dynamic optimization \citep{yazdani2021edosurvey}.

In many real applications, however, computing a better solution and deploying
it are fundamentally different operations. An operating schedule, a controller,
or a set of process parameters may already be running in production. Although a
new environment may admit a better solution, replacing the deployed one may
interrupt production, consume scarce resources, require human intervention, or
incur operational risks. Such costs, commonly referred to as
\emph{switching costs} \citep{yazdani2024review}, imply that continuously
redeploying the current optimum is not necessarily desirable. Consequently, the
decision problem extends beyond identifying the best solution for the present
environment and naturally raises a second question: \emph{how long can the
current solution remain acceptable before replacement becomes necessary?}

This observation motivated the development of \emph{Robust Optimization over
Time} (ROOT), whose objective is not merely to maximize instantaneous fitness
but to identify solutions that remain useful across successive environmental
changes \citep{yu2010robust,jin2013framework}. In contrast to classical robust
optimization, which typically studies uncertainty around a fixed optimization
problem, ROOT explicitly considers the temporal evolution of the problem
itself. Within this framework, \citet{jin2013framework} proposed constructing
surrogate objective functions by combining information from past environments
with predictions of future ones, allowing evolutionary algorithms to search for
solutions that balance present and future performance. Naturally, such
approaches rely on some degree of temporal continuity; when consecutive
environments become completely unrelated, historical information offers little
benefit and restarting the search may be unavoidable.

Different notions of temporal robustness have subsequently emerged.
\citet{fu2015robust} distinguish two complementary perspectives. The first
evaluates a solution through its \textit{average fitness} over a predetermined time
window. The second fixes an acceptability threshold and measures the number of
consecutive environments during which the solution remains acceptable before it
must be replaced. This duration, termed \textit{survival time}, directly reflects
the operational lifetime of a deployed solution. Since then, numerous ROOT
methods have sought solutions with long survival time by balancing robustness,
average fitness, and switching cost
\citep{guo2014twolayer,huang2017switching,huang2020decision}, estimating future
utility through sampling \citep{adam2019simple}, learning robust regions of the
search space \citep{yazdani2019learning,yazdani2023regions}, or incorporating
alternative search mechanisms such as differential evolution, reinforcement
learning, and data-stream models
\citep{guzman2020survival,lee2025rlpso,liu2025datastream}. These ideas have
also been successfully applied to engineering problems in which solution
lifetime is as important as instantaneous performance
\citep{fang2023modified,dong2024formation}.

Despite this growing body of work, existing research has concentrated almost
exclusively on \emph{how to search} for temporally robust solutions. Much less
attention has been devoted to understanding the behavior of the robustness
criterion itself once a solution has been selected. In particular, survival
time is usually evaluated empirically along realized environmental sequences,
whereas little is known about its expected behavior under stochastic
environmental dynamics. As a consequence, there currently exists no analytical
framework capable of answering questions such as how long a deployed solution
is expected to remain acceptable, how this lifetime depends on environmental
volatility, dimensionality, or initial solution quality, or under what
conditions a desired deployment horizon can be guaranteed.

This paper addresses this theoretical gap by studying survival time itself.
Specifically, we derive analytical bounds on the expected survival time of a
fixed solution under stochastic environmental dynamics and investigate how
deployment quality, environmental volatility, and problem dimension shape
its evolution. The resulting framework provides a theoretical
characterization of the deployment phase, complementing the predominantly
algorithmic ROOT literature with analytical insight into the persistence of
deployed solutions.

The remainder of the paper is organized as follows. Section
\ref{sec:model} introduces the stochastic model underlying the analysis.
Section~\ref{sec:results} presents the theoretical bounds and discusses their
properties. Section~\ref{sec:experiments} evaluates the theoretical predictions
through simulation and illustrates their implications for deployment decisions.
Finally, Section~\ref{sec:concluding-remarks} discusses the scope and limitations of
the proposed analysis and outlines future research directions.

\section{ROOT deployment model}\label{sec:model}

To study survival time analytically, it must first be expressed as a stochastic process. This section introduces the mathematical model used throughout the paper. After defining the deployment setting and the threshold-based notion of survival, we model environmental evolution as an isotropic Gaussian random walk and show that the resulting lifetime is a discrete first-exit problem. The section concludes by deriving the probabilistic identities required for the theoretical analysis presented in Section~\ref{sec:results}. 

\subsection{Deployment model and survival time}

The deployment setting is modeled by relating solution acceptability to the
distance between the deployed solution and the environmental optimum. Under a
threshold-based notion of robustness, a solution remains acceptable as long as
the optimum stays within a prescribed neighborhood of the deployed solution.
This geometric interpretation naturally recasts survival time as a discrete
first-exit problem.

Environments are indexed by \(t=0,1,\ldots\). Their changing parameter is the
location \(c_t\in\mathbb R^d\) of the environmental optimum. We assume that
this location follows an isotropic Gaussian random walk:
\begin{equation}
 c_{t+1}=c_t+\xi_t,\qquad
 \xi_t\stackrel{\mathrm{i.i.d.}}{\sim}\mathcal N(0,\sigma^2 I_d),
 \quad \sigma>0.
 \label{eq:random-walk}
\end{equation}
Here, \(\sigma\) controls the typical environmental change in each
coordinate. The Gaussian assumption provides a natural first-order model
when environmental changes arise from the accumulation of many small
perturbations with finite variance, while isotropy deliberately removes
directional preferences so that the analysis isolates the effects of
environmental volatility, dimensionality, and deployment error. The model
is therefore intended as a tractable baseline rather than a universal
description of ROOT environments.

Let \(x\) denote the single deployed solution. It is \emph{frozen}, meaning
that the optimizer may have produced \(x\), but \(x\) is not updated during
the lifetime analyzed below. In environment \(t\), consider the spherical
fitness
\begin{equation}
 f^{(t)}(x)=-\lVert x-c_t\rVert^2.
 \label{eq:fitness}
\end{equation}
The quality requirement \(f^{(t)}(x)\ge V\), for \(V\le0\), is equivalent to
\(\lVert x-c_t\rVert\le\delta\) with \(\delta=\sqrt{-V}\). More generally,
\(\delta\) is the tolerance induced by a monotone radial fitness threshold.
It is fixed by the application or decision maker, and a larger \(\delta\)
admits more performance degradation before replacement.

The deployment error left by the optimizer at deployment is
\(\epsilon_0=\lVert x-c_0\rVert\). We initially require
\(\epsilon_0<\delta\); otherwise the solution is unacceptable before any
future change occurs. Following the threshold-based definition of
\citet{fu2015robust}, the future survival time is
\begin{equation}
 \tau=\inf\{t\ge1:\lVert x-c_t\rVert>\delta\}.
 \label{eq:survival}
\end{equation}
Thus \(\tau=1\) means that the first environmental change makes \(x\)
unacceptable, while \(\tau=10\) means that the first threshold violation
occurs at the tenth change. Our convention excludes the already observed
environment \(t=0\), so \(\tau\ge1\). It differs by this indexing choice from
definitions that count the deployment environment itself, but represents the
same operational duration after deployment. In the terminology of
\citet{yazdani2024review}, this is a \(\mathrm{ROOT}^{S}_{Q}\) lifetime
analysis. It does not yet choose the next solution after failure and therefore
does not constitute a complete ROOT policy.

\subsection{Distance process and probabilistic foundations} 

The analysis is naturally expressed in terms of the displacement
\(Y_t=x-c_t\). Its squared norm,
\(D_t=\lVert Y_t\rVert^2\), measures the proximity of the deployed solution to
the acceptability boundary. Define the initial squared safety margin as
\[
\Delta=\delta^2-\epsilon_0^2>0,
\]
which quantifies the remaining tolerance before the threshold is reached. Expanding \(D_{t+1}\) yields
\begin{align}
 \mathbb E[D_{t+1}-D_t\mid\mathcal F_t]&=d\sigma^2,
 \label{eq:drift}\\
 \Var(D_{t+1}-D_t\mid\mathcal F_t)&=4D_t\sigma^2+2d\sigma^4,
 \label{eq:variance}
\end{align}
where \(\mathcal F_t=\sigma(c_0,\ldots,c_t)\). 
The variance expression follows by expanding the square and using the symmetry of the Gaussian increments; the cross terms vanish in expectation.

Equation~\eqref{eq:drift} shows that the squared distance increases, on
average, by \(d\sigma^2\) per environmental change. Consequently, for fixed
per-coordinate variance, the expected outward drift grows linearly with the
dimension. Equation~\eqref{eq:variance} characterizes the corresponding
conditional fluctuations, which depend on both the current displacement and the environmental variance.

The drift identity forms the basis of the lower bound. The upper bound
requires controlling the probability of remaining inside the acceptance
region over multiple environmental changes. To this end, we consider blocks
of \(m\) steps and define
\begin{align}
 a_m&=\frac{\delta^2}{m\sigma^2},&
 \lambda_m(r)&=\frac{r^2}{m\sigma^2},\label{eq:am-lambda}\\
 q_m&=F_{\chi^2_d}(a_m),&
 r_m(r)&=F_{\chi'^2_d(\lambda_m(r))}(a_m),\label{eq:q-r}
\end{align}
where \(F_{\chi^2_d}\) and \(F_{\chi'^2_d(\lambda)}\) are central and non-central chi-square distribution functions. \(r_m(r)\) is the probability that a Gaussian walk starting at
distance \(r\) lies inside the acceptance ball after \(m\) steps. Since the
process may leave and subsequently re-enter the ball, \(r_m(r)\) is an
endpoint probability rather than a survival probability.

The distinction between endpoint and path events is fundamental. Survival
through an entire block implies that the endpoint remains inside the
acceptance region, but the converse need not hold. Replacing the path event
by the corresponding endpoint event yields the following upper bound.

\begin{lemma}[Endpoint and block bounds]\label{lem:block}
For every \(y\) with \(\lVert y\rVert=r\le\delta\),
\begin{equation}
 \Pr_y(\tau>m)\le r_m(r)\le q_m<1.
 \label{eq:block-bound}
\end{equation}
Here \(q_m\) and \(r_m(r)\) are defined in \eqref{eq:q-r}; they depend on the block length \(m\), the dimension \(d\), the tolerance \(\delta\), and the noise variance \(\sigma^2\). 
Moreover, from any point in the acceptance ball,
\begin{equation}
 \Pr_y(\tau>km)\le q_m^k,\qquad k=0,1,\ldots.
 \label{eq:geometric-blocks}
\end{equation}
\end{lemma}

\begin{proof}
Survival through step \(m\) implies that the endpoint remains in the ball, so
\[
 \Pr_y(\tau>m)\le
 \Pr\!\left(\lVert y-\textstyle\sum_{j=0}^{m-1}\xi_j\rVert\le\delta\right)
 =r_m(r).
\]
The centered isotropic Gaussian density is symmetric and unimodal, and the
ball is symmetric and convex. Anderson's inequality therefore states that its
Gaussian measure is maximized when the translating vector is zero
\citep{anderson1955integral}; hence \(r_m(r)\le r_m(0)=q_m\).
Conditioning successively at times \(m,2m,\ldots\) and applying the same
uniform bound proves \eqref{eq:geometric-blocks}.
\end{proof}

The martingale argument used below requires the exit time and the exit state
to be integrable. The geometric tail established in
Lemma~\ref{lem:block} provides the necessary control.

\begin{proposition}[Exit, integrability, and Wald identity]\label{prop:wald}
The exit time is almost surely finite, \(\mathbb E[\tau]<\infty\),
\(\mathbb E[D_\tau]<\infty\), and
\begin{equation}
 d\sigma^2\mathbb E[\tau]=\mathbb E[D_\tau]-\epsilon_0^2.
 \label{eq:wald}
\end{equation}
\end{proposition}

\begin{proof}
Lemma~\ref{lem:block} with \(m=1\) yields
\(\Pr(\tau>k)\le q_1^k\), so \(\tau\) is finite almost surely and integrable.
To justify integrability at exit, note pathwise that
\[
 D_\tau\le(\delta+\lVert\xi_{\tau-1}\rVert)^2.
\]
For \(s=1,2\), Cauchy--Schwarz and the geometric tail give
\[
 \mathbb E[\lVert\xi_{\tau-1}\rVert^s]
 \le \sum_{k\ge1}
 \bigl(\mathbb E\lVert\xi_0\rVert^{2s}\bigr)^{1/2}
 \Pr(\tau\ge k)^{1/2}<\infty.
\]
Thus \(D_\tau\) is integrable independently of the identity to be proved.

The process \(M_t=D_t-d\sigma^2t\) is a martingale by
\eqref{eq:drift}. Optional stopping at the bounded time \(\tau\wedge n\)
gives
\[
 \mathbb E[D_{\tau\wedge n}]-\epsilon_0^2
 =d\sigma^2\mathbb E[\tau\wedge n].
\]
Monotone convergence applies on the right. On the left,
\(D_{\tau\wedge n}\le\max\{\delta^2,D_\tau\}\), so dominated convergence
establishes \eqref{eq:wald}.
\end{proof}

Equation~\eqref{eq:wald} links the expected survival time to the expected squared displacement at exit. This identity constitutes the key analytical bridge between the stochastic deployment model and the survival-time bounds derived in the next section.

\section{Analytical survival-time bounds}\label{sec:results}

The stochastic formulation of Section~\ref{sec:model} reduces survival-time
analysis to the study of a first-exit process. Building on the drift identity
of Proposition~\ref{prop:wald} and the block probability estimate of
Lemma~\ref{lem:block}, we derive computable lower and upper bounds for the
expected survival time. We then examine the dependence of these bounds on the
deployment error, environmental volatility, and problem dimension, and
discuss the conditions under which they provide useful guidance for
deployment decisions.

\subsection{A lower bound from the safety margin}

The lower bound follows directly from Proposition~\ref{prop:wald}.
Exit occurs only after the squared distance exceeds the acceptance threshold;
therefore, the expected exit time is bounded below by the initial squared
safety margin divided by the expected one-step increase in squared distance.
The discrete nature of the process further implies that every initially
acceptable deployment survives at least one environmental change.

\begin{theorem}[Lower bound]\label{thm:lower}
Under the model in Section~\ref{sec:model},
\begin{equation}
 \mathbb E[\tau]\ge
 L(\epsilon_0):=
 \max\left\{1,\frac{\delta^2-\epsilon_0^2}{d\sigma^2}\right\}.
 \label{eq:lower}
\end{equation}
\end{theorem}

\begin{proof}
The discrete floor follows from \(\tau\ge1\). At exit,
\(D_\tau>\delta^2\) almost surely. Substitution into
\eqref{eq:wald} yields the second term.
\end{proof}

Equation~\eqref{eq:lower} admits a direct operational interpretation. Given a
required deployment horizon \(h\), the condition
\(L(\epsilon_0)\ge h\) provides a sufficient guarantee on the expected
survival time. The bound also separates the contributions of the three model
parameters. Increasing the deployment quality reduces the deployment error
\(\epsilon_0\), increasing the acceptance threshold enlarges the safety margin
\(\delta\), whereas greater environmental volatility or higher dimensionality
reduces the guaranteed lifetime through the factor \(d\sigma^2\). Among these,
only \(\epsilon_0\) is directly controlled by the optimization algorithm,
while the remaining quantities are determined by the application and the
environment.

\subsection{A block-based upper bound}

Unlike the lower bound, the upper bound cannot be obtained directly from the
drift identity because survival depends on the entire trajectory rather than
only on the exit state. The block estimate of
Lemma~\ref{lem:block} provides the additional ingredient required to control
the survival-time tail. By partitioning time into blocks of \(m\) successive
environmental changes, we obtain the following family of computable upper
bounds. Here, \(m\) is an analytical parameter introduced solely for the
derivation and does not alter the underlying stochastic model.

\begin{theorem}[Multi-step upper bound]\label{thm:upper}
For every integer \(m\ge1\),
\begin{equation}
 \mathbb E[\tau]\le
 U_m(\epsilon_0):=
 m\left(1+\frac{r_m(\epsilon_0)}{1-q_m}\right).
 \label{eq:upper-m}
\end{equation}
Consequently,
\begin{equation}
 \mathbb E[\tau]\le
 U(\epsilon_0):=\inf_{m\in\mathbb N}U_m(\epsilon_0).
 \label{eq:upper}
\end{equation}
\end{theorem}

\begin{proof}
For a walk beginning at any point of the acceptance ball, grouping the
tail-sum formula into blocks and using \eqref{eq:geometric-blocks} gives
\[
 \sup_{\lVert y\rVert\le\delta}\mathbb E_y[\tau]
 \le m\sum_{k\ge0}q_m^k=\frac{m}{1-q_m}.
\]
For the specified initial displacement,
\[
 \mathbb E[\tau]
 =\mathbb E[\min(\tau,m)]+\mathbb E[(\tau-m)_+].
\]
The first term is at most \(m\). Conditional on survival through \(m\), the
remaining expected lifetime is at most \(m/(1-q_m)\), while
\(\Pr(\tau>m)\le r_m(\epsilon_0)\) by Lemma~\ref{lem:block}. This proves
\eqref{eq:upper-m}; taking the infimum proves \eqref{eq:upper}.
\end{proof}

Equation~\eqref{eq:upper} complements the lower bound by providing a
necessary condition for achieving a prescribed expected deployment horizon.
Given a target horizon \(h\), the inequality
\(U(\epsilon_0)<h\) rules out its attainment under the assumed stochastic
model. Together, the bounds partition the decision into three regions:
guaranteed feasibility (\(L(\epsilon_0)\ge h\)),
guaranteed infeasibility (\(U(\epsilon_0)<h\)), and an intermediate regime
\(L(\epsilon_0)<h\le U(\epsilon_0)\) where the available analytical
information is insufficient to determine the outcome. This indeterminate
region naturally identifies situations in which refined analysis, improved
parameter estimation, or simulation may be justified.

\begin{remark}
  The optimization in \eqref{eq:upper} is one-dimensional over the positive
  integers. Since \(U_m\ge m\), the search can be terminated once
  \(m\) exceeds the smallest upper bound found so far. Consequently,
  evaluating \(U(\epsilon_0)\) requires neither trajectory simulation nor
  continuous optimization.
  \end{remark}

\subsection{Asymptotic regimes}
The lower and upper bounds jointly characterize the asymptotic behavior of
the expected survival time. Two regimes are of particular interest: vanishing
environmental volatility and increasing problem dimension under fixed
environmental volatility. These regimes are characterized by the following
two propositions.

\begin{proposition}[Low-noise order]\label{prop:low-noise}
Fix \(d,\delta\), and \(\rho=\epsilon_0/\delta<1\). As
\(\sigma\downarrow0\),
\begin{equation}
 \mathbb E[\tau]=\Theta(\sigma^{-2}).
 \label{eq:low-noise}
\end{equation}
The lower-bound scale is
\((1-\rho^2)\delta^2/(d\sigma^2)\).
\end{proposition}

\begin{proof}
The lower order follows from Theorem~\ref{thm:lower}. For the upper order,
choose \(m_\sigma=\lceil\delta^2/(d\sigma^2)\rceil\) in
Theorem~\ref{thm:upper}. Then \(m_\sigma=\Theta(\sigma^{-2})\),
\(a_{m_\sigma}\to d\), and
\(\lambda_{m_\sigma}(\epsilon_0)\to\rho^2d\). Since
\(1-F_{\chi^2_d}(d)>0\), convergence implies that the actual denominator
\(1-F_{\chi^2_d}(a_{m_\sigma})\) is bounded below by a positive constant for
all sufficiently small \(\sigma\). The numerator distribution function is at
most one. Hence \(U_{m_\sigma}=O(\sigma^{-2})\).
\end{proof}

Proposition~\ref{prop:low-noise} establishes the exact asymptotic order,
whereas the multiplicative constant remains determined by the deployment
error, acceptance threshold, and problem dimension. The finite-\(\sigma\)
behavior is therefore governed by these parameters rather than by the
asymptotic scaling itself.

A second asymptotic regime arises when the problem dimension increases while
the environmental volatility remains fixed. Under the isotropic model, the
expected squared displacement accumulated in a single environmental change
grows linearly with \(d\), causing the acceptance region to be exited
increasingly rapidly.

\begin{proposition}[High-dimensional collapse]\label{prop:high-d}
For fixed \(\epsilon_0<\delta\), \(\delta\), and \(\sigma>0\),
\begin{equation}
 \lim_{d\to\infty}\mathbb E[\tau]=1.
 \label{eq:high-d}
\end{equation}
The convergence is exponentially fast for all sufficiently large \(d\).
\end{proposition}

\begin{proof}
Set \(m=1\). Since \(r_1(\epsilon_0)\le q_1\),
Theorem~\ref{thm:upper} gives
\[
 1\le\mathbb E[\tau]\le\frac{1}{1-q_1},\qquad
 q_1=\Pr\!\left(\chi^2_d\le\frac{\delta^2}{\sigma^2}\right).
\]
For fixed \(\delta/\sigma\), the central chi-square lower tail tends to zero.
More precisely, once \(d\) is large enough that
\(\delta^2/\sigma^2 \le d/4\), the Laurent--Massart inequality \citep{laurent2000adaptive} yields
\(q_1\le\exp(-9d/64)\). Therefore
\(\mathbb E[\tau]-1\le q_1/(1-q_1)\), which is exponentially small.
\end{proof}

Proposition~\ref{prop:high-d} strengthens the dimensional dependence already
suggested by the lower bound. Rather than indicating a \(1/d\) scaling, it
shows that the expected survival time converges exponentially fast to its
minimum attainable value of one environmental change. Under the isotropic
model, this collapse is driven entirely by the environmental dynamics and
cannot be prevented by reducing the initial deployment error.

\subsection{Optimal deployment under isotropic dynamics}

The previous results quantify the expected lifetime of any fixed deployment.
A natural question is whether the deployment maximizing survival time must
coincide with the current optimum, or whether deliberately sacrificing
instantaneous fitness could increase future persistence. Under the isotropic
Gaussian model considered here, the answer is negative. The symmetry of both
the environmental dynamics and the acceptance region implies that every
displacement from the current optimum can only reduce the expected survival
time.

\begin{corollary}[Optimality of the current optimum]
\label{cor:no-displacement}
Both \(L(\epsilon_0)\) and \(U(\epsilon_0)\) are non-increasing in
\(\epsilon_0\). For \(U\), the monotonicity follows because for each \(m\), the non-central chi-square probability \(r_m(\epsilon_0)\) decreases with the non-centrality parameter \(\lambda_m(\epsilon_0)\), while \(q_m\) is independent of \(\epsilon_0\). Moreover, \(x=c_0\) maximizes the actual expected survival
time over all initially acceptable fixed solutions.
\end{corollary}

\begin{proof}
The lower bound is immediate. For each \(m\), the non-central chi-square
probability \(r_m(\epsilon_0)\) decreases as its non-centrality parameter
increases, while \(q_m\) is independent of \(\epsilon_0\). Thus every
\(U_m\), and therefore their infimum, is non-increasing. For the actual
expectation, fix \(n\) and write
\(Z=(S_1,\ldots,S_n)\), where
\(S_k=\sum_{j=0}^{k-1}\xi_j\). The event
\(\{\tau>n\}\) is a translation, by
\((Y_0,\ldots,Y_0)\), of the symmetric convex set
\[
 K_n=\{(z_1,\ldots,z_n):\lVert z_k\rVert\le\delta
 \text{ for every }k\}.
\]
The vector \(Z\) is centered Gaussian. Anderson's inequality gives
\(\Pr_{Y_0}(\tau>n)\le\Pr_0(\tau>n)\). Summing this inequality over \(n\)
using the tail-sum formula proves the claim.
\end{proof}

This result clarifies the scope of the present theory. Under the isotropic
model, ROOT does not create a trade-off between instantaneous fitness and
future persistence: the current optimum is already the deployment that
maximizes the expected survival time. Consequently, the role of the analysis
is not to identify a better deployment than TMO, but to quantify how long the
best available deployment is expected to remain acceptable and whether that
lifetime satisfies operational requirements.

The conclusion is model-dependent rather than universal. Directional drift,
anisotropic or correlated environmental changes, asymmetric acceptance
regions, or multimodal landscapes may shift the survival-optimal deployment
away from the current optimum. Such extensions fall outside the present
analysis and constitute natural directions for future work.

\section{Computational validation and deployment analysis}
\label{sec:experiments}

This section evaluates the practical behavior of the proposed analytical
framework. After describing the common experimental protocol, we first
validate the analytical bounds over a broad factorial design. We then
examine their robustness through sensitivity analyses involving anisotropic
environmental dynamics and uncertainty in the estimated environmental
parameters. Next, we illustrate how the bounds can be used as a
post-optimization decision layer following evolutionary search. Finally, we
embed the framework within a sequential deployment scenario to demonstrate
its role across repeated environmental changes.

\subsection{Experimental settings}
\label{sec:experimental-settings}

Unless otherwise stated, all experiments use an acceptance radius
\(\delta=1\). The main validation considers the Cartesian product of
\(\sigma\in\{0.05,0.10,0.20\}\),
\(\epsilon_0/\delta\in\{0.1,0.3,0.5,0.7,0.9\}\), and
\(d\in\{2,5,10,20,50,100,200,500\}\), yielding 120 parameter
configurations. Each configuration is evaluated using \(30{,}000\)
independent Monte Carlo trajectories, with uncertainty reported as an
approximate 95\% normal confidence interval for the sample mean.

The parameter ranges cover slow, intermediate, and rapidly changing
environments; initial deployments from nearly optimal to close to the
acceptability threshold; and dimensions spanning the transition from
persistent survival to near-immediate failure. Unless stated otherwise, the
upper bound is evaluated by increasing the block length \(m\) until
\(U_m\) exceeds the best value already obtained, following the stopping
criterion discussed after Theorem~\ref{thm:upper}. All code, random seeds,
and raw experimental outputs are archived as described in the Code
Availability statement.

The decision-support, estimation-sensitivity, and sequential-redeployment
experiments additionally require a controlled generator of deployment
candidates. Throughout those studies we use a standard global-best particle
swarm optimizer (PSO)~\citep{kennedy1995particle} with Clerc--Kennedy
constriction~\citep{clerc2002particle}. The swarm optimizes only the current
environment \(f^{(0)}(x)=-\lVert x\rVert^2\) over
\([-5.12,5.12]^d\), with 30 particles, constriction coefficient
\(\chi=0.7298\), acceleration coefficients \(c_1=c_2=2.05\), and velocity
and position clipped to the search domain. Evaluation budgets of 30, 150,
600, and 3000 function evaluations are considered where relevant, and
thirty independent optimization runs are performed for each reported budget.
The PSO is used solely as a candidate generator; no claim of algorithmic
superiority is made. Experiment-specific protocols are stated in the
corresponding subsections below.

The computational study tests four theoretical expectations. First, the
expected survival time should scale as
\(\Theta(\sigma^{-2})\) in the low-noise regime. Second, increasing the
deployment error \(\epsilon_0\) should reduce the expected lifetime. Third,
the optimal block length \(m^*\) should approach one as the one-step drift
\(d\sigma^2\) increases. Finally, the impact of parameter uncertainty should
be greatest when a required deployment horizon lies close to one of the two
analytical bounds.

\subsection{Global validation and behavior of the analytical bounds}

The analytical interval \([L,U]\) is evaluated over the complete factorial
design. Each configuration is summarized by two metrics: the relative
interval width \(W=(U-L)/E_{\mathrm{MC}}\) and the relative position
\(R=(E_{\mathrm{MC}}-L)/(U-L)\). These quantities characterize the
informativeness of the analytical bracket and the location of the Monte
Carlo estimate within it.

\begin{figure}[ht]
  \centering
  \includegraphics[width=\textwidth]{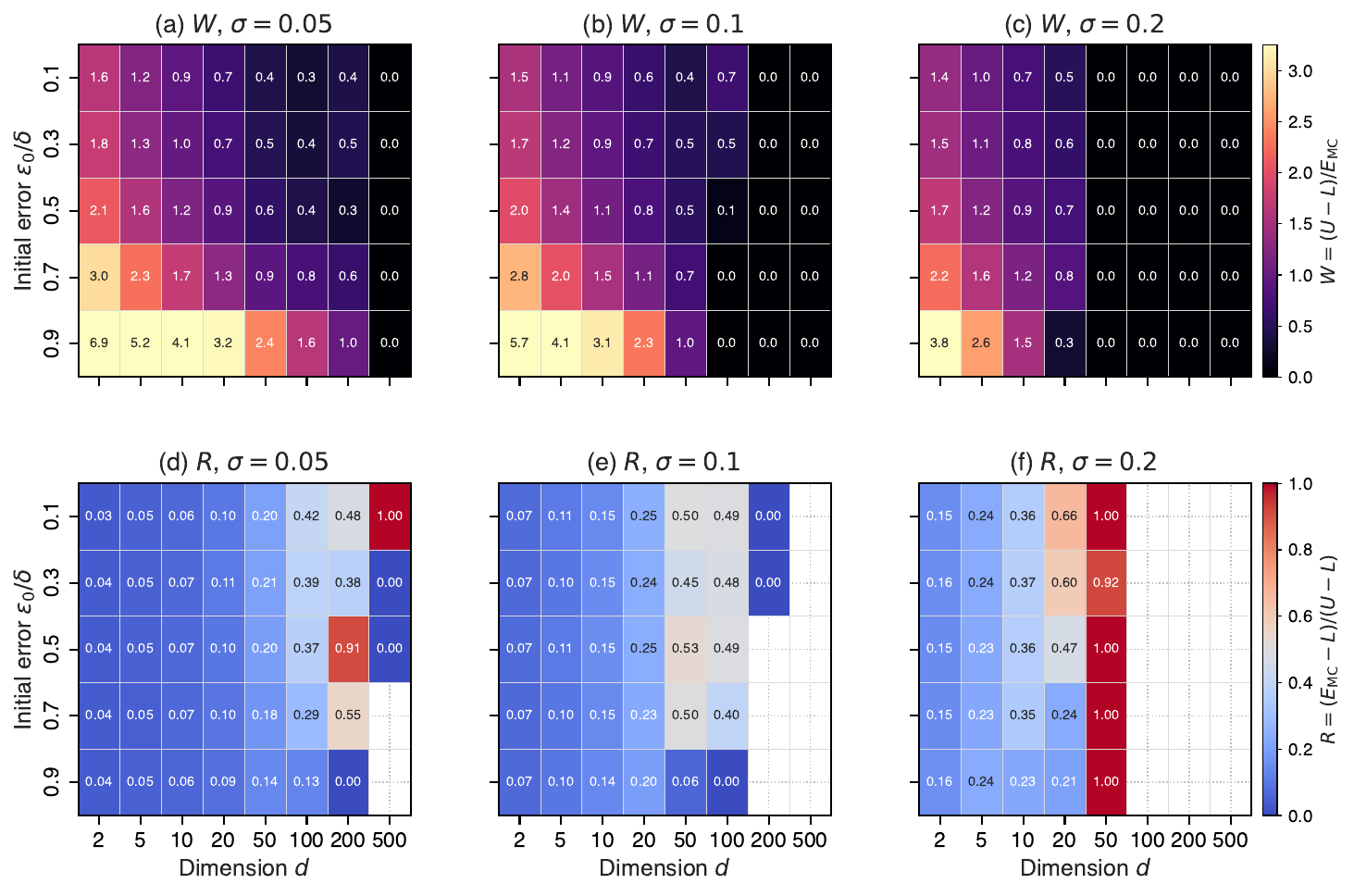}
  \caption{Practical quality of the analytical interval over all 120
  configurations. Top row: relative width
  \(W=(U-L)/E_{\mathrm{MC}}\) for
  \(\sigma=0.05\), \(0.10\), and \(0.20\) (panels a--c). Bottom row:
  relative position \(R=(E_{\mathrm{MC}}-L)/(U-L)\) on the same grid
  (panels d--f), with a diverging scale centred at \(0.5\). Cells with
  \(U=L\) leave \(R\) undefined (shown as ``--''). Colour scales are shared
  within each row; a few one-step means that sit marginally above \(U\) are
  clipped to \(R=1\) for display.}
  \label{fig:bounds-quality}
  \end{figure}

Figure~\ref{fig:bounds-quality} provides a global view of the analytical
interval. The relative width \(W\) exhibits a clear transition across the
parameter space. The interval is widest in low-dimensional, low-volatility
settings, particularly when the deployed solution starts close to the
acceptability boundary. In these regimes, the endpoint relaxation used to
derive the upper bound is necessarily conservative because trajectories may
leave and subsequently re-enter the acceptance region. As either the
dimension or the environmental volatility increases, the interval rapidly
contracts until it becomes almost indistinguishable from the Monte Carlo
estimate in the one-step regime.

The relative position \(R\) reveals a complementary pattern. For most configurations, the Monte Carlo mean lies substantially closer to the lower bound than to the upper one; the lower bound captures most of the
practically relevant information, while the upper bound remains intentionally
conservative. Only near the transition to immediate failure does the estimate
move toward the centre of the interval, and once the interval collapses
(\(U=L\)), the position metric becomes undefined. Across all configurations,
the Monte Carlo confidence intervals remain consistent with the analytical
interval. Five point estimates lie marginally above a numerically tight upper
bound in the one-step regime; their confidence intervals intersect
\([L,U]\), indicating deviations compatible with Monte Carlo sampling error
rather than with any systematic departure from the theory.

\begin{table}[ht]
  \caption{Selected validation results for
  \(\epsilon_0=0.5\delta\), \(\sigma=0.10\), and \(N=30{,}000\).
  The upper-bound column also reports its minimizing block length \(m^*\).}
  \label{tab:bounds}
  \centering
  \begin{tabular}{@{}rrrrr@{}}
  \toprule
  \(d\) & Lower & Monte Carlo mean & Upper & \(m^*\)\\
  \midrule
  2   & 37.500 & \(43.612\pm0.438\) & 123.637 & 46\\
  10  & 7.500  & \(8.985\pm0.051\)  & 17.241  & 12\\
  50  & 1.500  & \(2.107\pm0.005\)  & 2.640   & 2\\
  100 & 1.000  & \(1.068\pm0.003\)  & 1.138   & 1\\
  \bottomrule
  \end{tabular}
  \end{table}

Table~\ref{tab:bounds} provides detailed numerical values for the representative \(\epsilon_0=0.5\delta\) case, while Figure~\ref{fig:bounds} extends the comparison to other displacement settings. For low-dimensional problems (\(d=2\)), the lower bound captures the correct lifetime scale, whereas the upper bound remains conservative because endpoint containment is only a necessary condition for pathwise survival. As the dimension increases, the interval narrows substantially. By \(d=50\), the Monte Carlo estimate lies near the center of the interval and the optimal block length has decreased to \(m^*=2\). At \(d=100\), the analysis collapses to the one-step regime established in Proposition~\ref{prop:high-d} (\(m^*=1\)).

\begin{figure}[ht]
\centering
\includegraphics[width=0.8\textwidth]{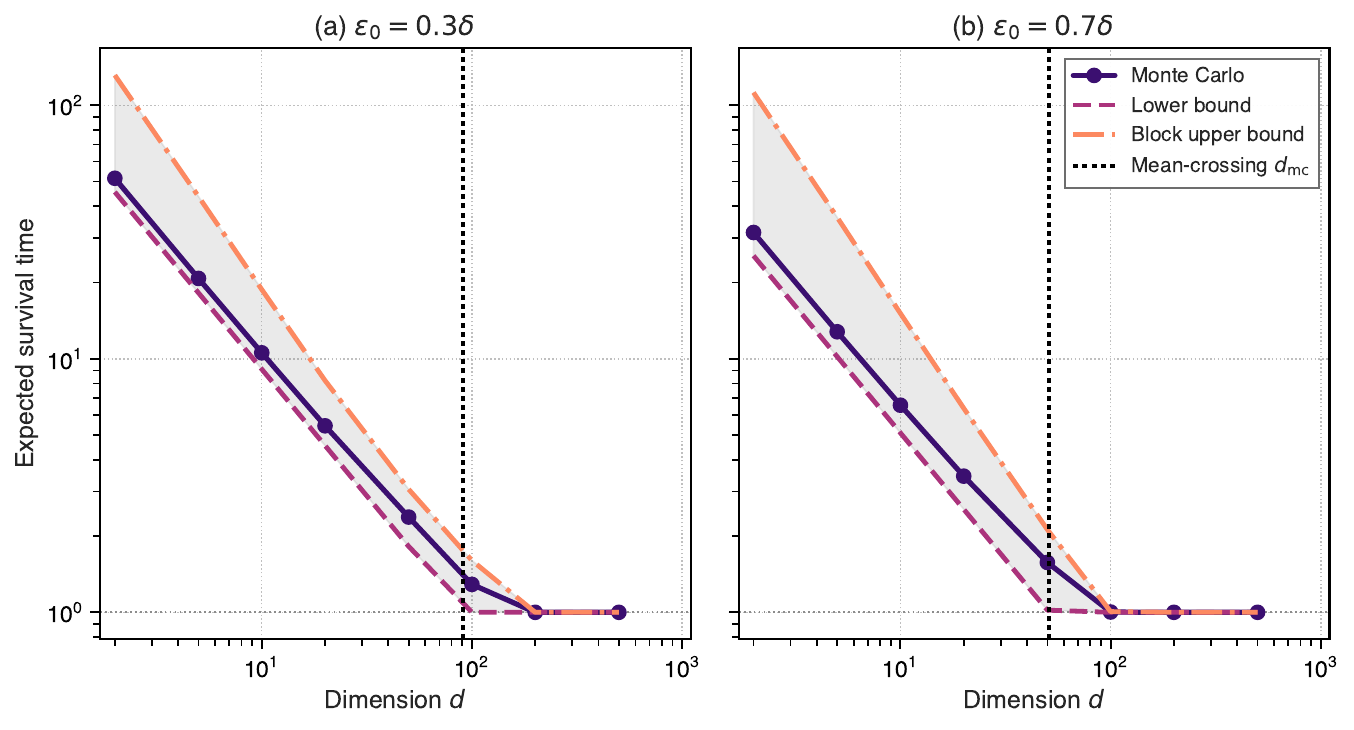}
\caption{Expected survival time as a function of dimension \(d\) on a
log-log scale (\(\sigma=0.10\), \(\delta=1\)). Panels (a) and (b)
correspond to initial displacements \(\epsilon_0=0.3\delta\) and
\(\epsilon_0=0.7\delta\). Solid lines with
markers: Monte Carlo mean over 30{,}000 trajectories, with a shaded 95\%
confidence band. Dashed and dotted lines: the lower bound~\eqref{eq:lower}
and the block upper bound~\eqref{eq:upper}; the shaded band between them
marks the analytical interval \([L,U]\). The dotted vertical line marks 
\(d_{\mathrm{mc}}=(\delta^2-\epsilon_0^2)/\sigma^2\), the dimension at 
which the deterministic one-step drift equals the initial squared safety 
margin. This quantity is not a theoretical threshold derived from 
Theorems~\ref{thm:lower} or~\ref{thm:upper}; it serves only as a visual 
reference for where the empirical survival curves depart from the 
approximate linear regime.}
\label{fig:bounds}
\end{figure}

Figure~\ref{fig:bounds} extends this comparison across the full dimensional
range. In both displacement settings, the empirical expectation follows the
analytical trend. Before the transition, the curves decay approximately
linearly on the log-log scale, reflecting the inverse-dimensional dependence
of the lower bound. As dimension increases, both the simulations and the
analytical interval progressively collapse towards the one-step limit
established in Proposition~\ref{prop:high-d}. The descriptive quantity
\(d_{\mathrm{mc}}=(\delta^2-\epsilon_0^2)/\sigma^2\) marks the dimension at
which the deterministic one-step drift equals the initial squared safety
margin; it consistently appears near the point where the empirical curves
depart from the approximately linear regime and begin converging towards
immediate exit.

The behavior of the minimizing block length \(m^*\) is shown in
Figure~\ref{fig:parameter-sensitivity}. As predicted by the analysis
in Section~\ref{sec:results}, \(m^*\) decreases monotonically with the
one-step drift \(d\sigma^2\): longer blocks become preferable when 
survival is governed by gradual accumulation, whereas the analysis 
automatically reverts to a one-step argument once environmental changes 
dominate.

\begin{figure}[t]
\centering
\includegraphics[width=0.6\textwidth]{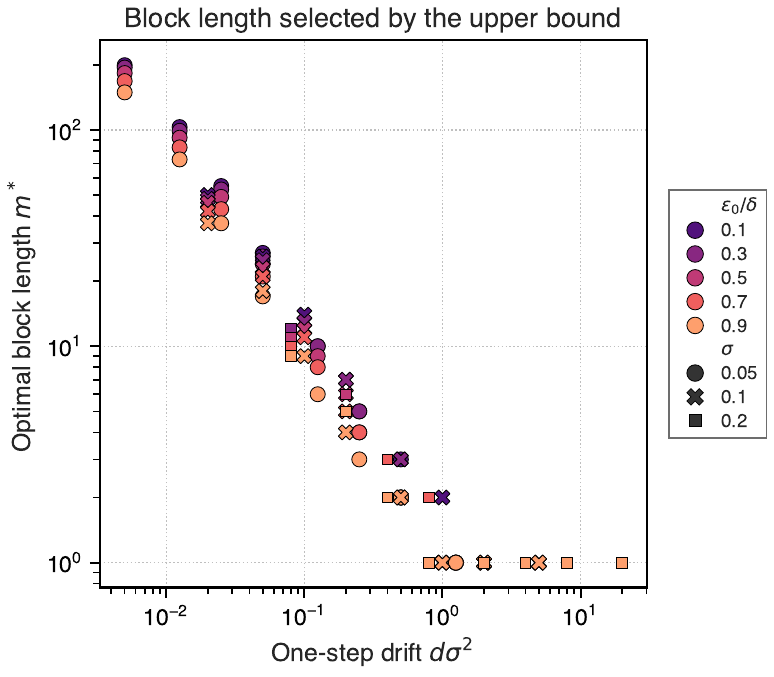}
\caption{Block length \(m^*\) selected by the upper bound as a function of
the one-step drift \(d\sigma^2\), pooled across all volatilities and
displacements.}
\label{fig:parameter-sensitivity}
\end{figure}

\subsection{Sensitivity analyses}

The theoretical results derived in Section~\ref{sec:results} rely on two
types of assumptions: assumptions about the environmental dynamics and
assumptions about the parameters supplied to the analytical decision layer.
To assess how strongly the practical conclusions depend on these inputs, we
consider two complementary sensitivity analyses. The first relaxes the
isotropy assumption while preserving the overall environmental variability,
whereas the second examines how errors in the estimated environmental
parameters affect the resulting deployment decisions.

\subsubsection{Sensitivity to anisotropic environmental noise}

The analytical results were derived under isotropic Gaussian drift, in which
all coordinates evolve with the same variance. To assess how strongly the
observed behavior depends on this assumption, we perform a controlled
sensitivity analysis by relaxing isotropy while preserving the total
environmental variability.

Specifically, the covariance matrix \(\sigma^2I_d\) is replaced by a
diagonal covariance matrix \(\Sigma\) whose coordinate variances are
proportional to values linearly spaced between 0.25 and 1.75 and then
normalized so that \(\operatorname{tr}(\Sigma)=d\sigma^2\). Matching the
trace preserves the expected one-step increase in squared distance, and
therefore the drift underlying the lower bound, while redistributing that
variability across different directions. The upper bound in
Theorem~\ref{thm:upper}, however, relies on isotropic Gaussian ball
probabilities and therefore is not expected to extend to this setting.

\begin{table}[t]
\caption{Isotropic versus trace-matched anisotropic dynamics
(\(\epsilon_0=0.5\delta\), \(\sigma_{\mathrm{eff}}=0.10\),
\(N=30{,}000\)). The last column gives the relative change of the anisotropic Monte Carlo mean with respect to the isotropic mean; all differences are within the Monte Carlo sampling variability (typical standard errors of the mean are about \(0.01\)–\(0.02\) in the last two rows, and about \(0.4\) for \(d=2\)).}
\label{tab:anisotropic}
\centering
\begin{tabular}{@{}rrrrr@{}}
\toprule
\(d\) & Lower~\eqref{eq:lower} & Isotropic \(\mathbb{E}[\tau]\) &
Anisotropic \(\mathbb{E}[\tau]\) & \(\Delta_{\%}\)\\
\midrule
2   & 37.500 & \(43.955\pm0.441\) & \(44.254\pm0.425\) & \(+0.68\%\)\\
10  & 7.500  & \(8.978\pm0.051\)  & \(8.979\pm0.047\)  & \(+0.01\%\)\\
50  & 1.500  & \(2.107\pm0.005\)  & \(2.108\pm0.004\)  & \(+0.04\%\)\\
100 & 1.000  & \(1.064\pm0.003\)  & \(1.053\pm0.003\)  & \(-1.09\%\)\\
\bottomrule
\end{tabular}
\end{table}

Table~\ref{tab:anisotropic} compares the isotropic and trace-matched
anisotropic dynamics for representative dimensions. The observed differences
in expected survival time remain small throughout, never exceeding
approximately \(1.1\%\). In particular, the anisotropic perturbation
produces only marginal increases in low and intermediate dimensions and a
slightly shorter lifetime in the highest-dimensional case. Across all
configurations, the drift-based lower bound remains below both Monte Carlo
estimates, as expected from its dependence solely on the total variance.

These results should be interpreted as a sensitivity analysis rather than as
an extension of the theoretical results. They do not justify applying the
upper bound to arbitrary covariance structures, since its derivation depends
essentially on isotropic Gaussian ball probabilities. They do suggest,
however, that the principal scaling captured by the lower bound is robust to
moderate redistributions of variance when the total environmental
variability is preserved. In other words, for the trace-matched perturbation
considered here, the dominant factor governing the expected deployment
lifetime is the overall drift magnitude rather than its distribution across
individual coordinates.

\subsubsection{Sensitivity to environmental parameter estimation}

The analytical bounds assume that the environmental parameters are known.
In practice, however, both the current optimum \(c_0\) and the
environmental volatility \(\sigma\) must typically be estimated from
observations. Since the deployment decision depends on these quantities, it
is important to assess how estimation errors propagate to the analytical
classification. Following the distinction between true and surrogate
environmental information discussed in the ROOT literature
\citep{jin2013framework,yazdani2024review}, we perturb only the parameters
supplied to the decision layer while keeping the simulated environmental
process unchanged.

The estimated centre is perturbed by isotropic Gaussian noise with expected
radial scales of \(0\), \(0.025\delta\), \(0.05\delta\), and
\(0.10\delta\). Independently, the estimated environmental volatility
receives zero-mean relative Gaussian noise with standard deviations of
\(0\), \(5\%\), \(10\%\), and \(20\%\). For each perturbation level, we
compare the deployment decision obtained from the perturbed parameters with
the decision obtained using the true parameters.

\begin{figure}[t]
\centering
\includegraphics[width=0.6\textwidth]{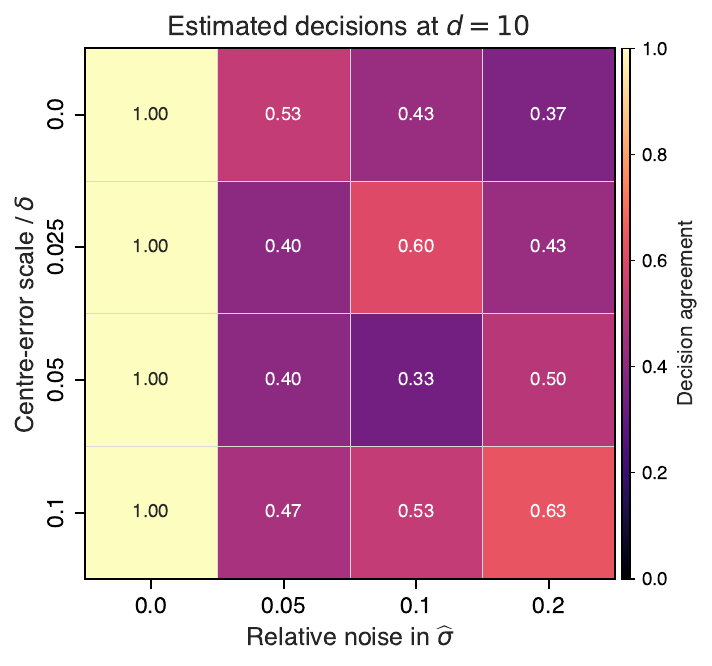}
\caption{Agreement between decisions made with true and perturbed
parameters for candidates from the PSO setup of
Section~\ref{sec:experimental-settings} (budget 3000) at \(d=10\).
Each heat-map cell aggregates 30 runs.}
\label{fig:estimation-sensitivity}
\end{figure}

Figure~\ref{fig:estimation-sensitivity} summarizes the agreement between the
deployment decisions obtained with the true and perturbed parameter values.
The dominant effect is not the magnitude of the perturbation itself but the
distance of the operating point from a decision boundary. At \(d=10\), where
the PSO candidates with budget 3000 lie close to the target horizon, the
agreement remains perfect when only the centre estimate is perturbed, since
the analytical classification is already \textit{undetermined}. In contrast, errors
in the estimated volatility alter both analytical bounds and therefore change
the classification more frequently, reducing the agreement to between 0.33
and 0.63 for the perturbation levels considered.

The behavior differs markedly in parameter regimes that are well separated
from the decision boundary. For example, at \(d=2\) the lower bound remains
comfortably above the required horizon, and every perturbation level retains
100\% agreement. The robustness of the analytical decision therefore depends 
primarily on the available margin to the relevant bound rather than on the 
absolute magnitude of the estimation errors.

These experiments should be interpreted as a sensitivity analysis rather
than as an uncertainty quantification procedure. The analytical theorems
remain conditional on the supplied parameter values, and replacing the true
parameters by point estimates does not account for estimation uncertainty.
Obtaining certified deployment decisions under uncertain environmental
parameters would require confidence regions for \(c_0\) and \(\sigma\),
followed by worst-case evaluation of the corresponding analytical bounds.

\subsection{Decision support after evolutionary search}

The analytical bounds are functions solely of the deployment error
\(\epsilon_0\), independent of its origin. They can therefore be applied
as a post-optimization decision layer on top of any search algorithm. To
illustrate this use, we employ the PSO of
Section~\ref{sec:experimental-settings} as a controlled generator of
deployment candidates with different optimization qualities. The purpose is
not to compare PSO with other optimizers, but to show how the analytical
interval translates the residual deployment error of a candidate into an
operational statement about its expected survival time. For each
optimization run, the best candidate is frozen and subsequently evaluated
over 3000 independent environmental trajectories. Future environments are
never exposed to the optimizer; they are generated only after optimization
has finished.

Given a required expected deployment horizon \(h=10\), the decision
framework of Section~\ref{sec:results} classifies each candidate according
to the analytical interval \([L,U]\). Candidates outside the acceptance
region (\(\epsilon_0\ge\delta\)) are reported separately as
\emph{rejected at deployment}, distinguishing failure to obtain an
acceptable solution from failure to guarantee the required operational
lifetime. Throughout this experiment, ``certified'' refers exclusively to
the analytical guarantees established in
Section~\ref{sec:results}; the chi-square probabilities
\(F_{\chi^2_d}\) and \(F_{\chi'^2_d(\lambda)}\) defined in
Section~\ref{sec:model} are evaluated in the log domain to avoid underflow.

\begin{figure}[t]
\centering
\includegraphics[width=0.65\textwidth]{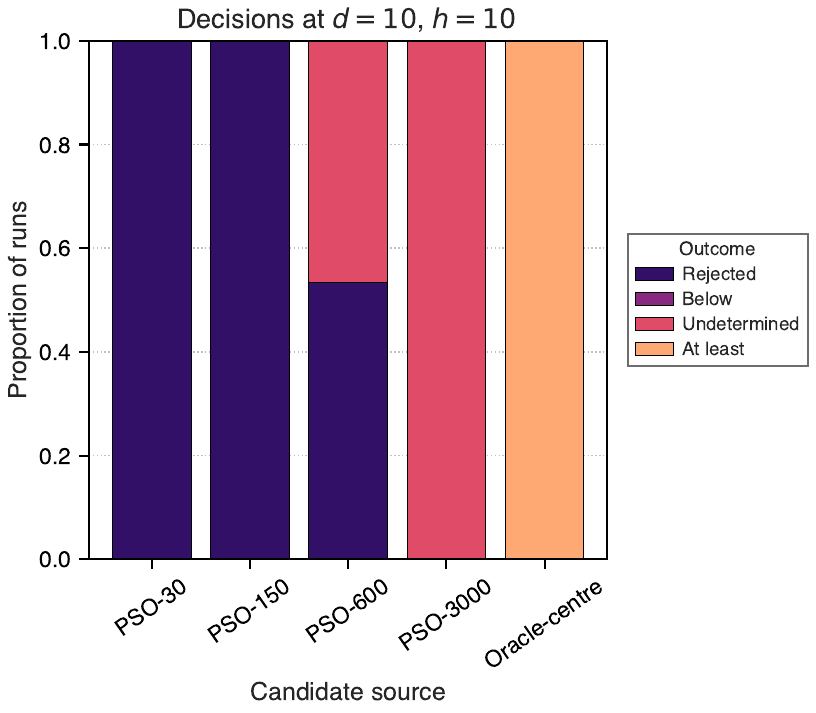}
\caption{Decision outcomes at \(d=10\) for a target expected horizon
\(h=10\). Rejection at deployment is separated from bound-based decisions
as the PSO evaluation budget increases; the oracle centre is included as a
reference.}
\label{fig:decision-outcomes}
\end{figure}

Figure~\ref{fig:decision-outcomes} illustrates how the analytical interval
changes the interpretation of optimization results. At \(d=10\), low
evaluation budgets frequently fail to produce candidates that are acceptable
at deployment. As the optimization budget increases, the deployment error
decreases and all candidates eventually satisfy the acceptance criterion.
However, acceptability alone does not imply that the required operational
lifetime is guaranteed. Even with 3000 evaluations, the empirical mean
survival exceeds ten environmental changes, but the lower bound remains just
below the required threshold
(\(\overline{L}=9.999<10\)). This near-boundary case was deliberately chosen to illustrate the decision boundary; it highlights how small changes in deployment error can shift the classification. The decision therefore falls into the 
undetermined region of the analytical framework (Section~\ref{sec:results}): 
the lower theorem cannot certify the horizon, even though simulation 
suggests it is often achieved. The oracle solution (\(\epsilon_0=0\)) 
highlights how narrow this decision boundary is, consistent with 
Corollary~\ref{cor:no-displacement}: under isotropic dynamics, eliminating 
the deployment error is both sufficient and necessary for optimal survival.

The complementary situation appears at \(d=50\). Even the oracle solution
cannot satisfy the required lifetime because the analytical upper bound is
\(U(0)=3.104<10\). The limitation is therefore environmental rather than
algorithmic: no optimizer capable only of improving the current solution can
achieve the requested deployment horizon under the assumed dynamics. In such
cases, the appropriate decision is not to continue searching, but to revise
the deployment requirements, increase the acceptance tolerance, reduce the
environmental volatility, or adopt a different environmental model or an
adaptive replacement strategy. This experiment illustrates the intended role
of the proposed framework as a decision-support layer that complements
optimization by distinguishing insufficient search effort from
fundamentally unattainable deployment objectives.

\subsection{Sequential deployment analysis}

The preceding experiments evaluate individual deployment decisions under
controlled parameter settings. In practice, however, deployment decisions
are made repeatedly as the environment evolves. The analytical framework is
therefore intended to operate not once, but throughout a sequence of
redeployments. This experiment illustrates that usage by embedding the
proposed decision layer within a longitudinal ROOT scenario and examining
how it guides successive deployment decisions over time.

The environment follows the same isotropic sphere model used throughout this
section (\(d=10\), \(\delta=1\), \(\sigma=0.10\)) over a horizon of
200 environmental changes. Two deployment policies are compared. The TMO
baseline performs a new optimization and redeployment after every
environmental change, irrespective of whether the incumbent solution
remains acceptable. In contrast, the
\(\mathrm{ROOT}^{S}_{Q}\) policy keeps the deployed solution until the
quality threshold is violated, at which point the PSO of
Section~\ref{sec:experimental-settings} with budget 3000 generates a
replacement candidate. After each redeployment,
the analytical bounds are recomputed from the new candidate's deployment error.
Since both policies use exactly the same optimizer and search
budget, the comparison isolates the effect of the deployment policy rather
than the optimization algorithm itself.

\begin{figure}[t]
\centering
\includegraphics[width=\textwidth]{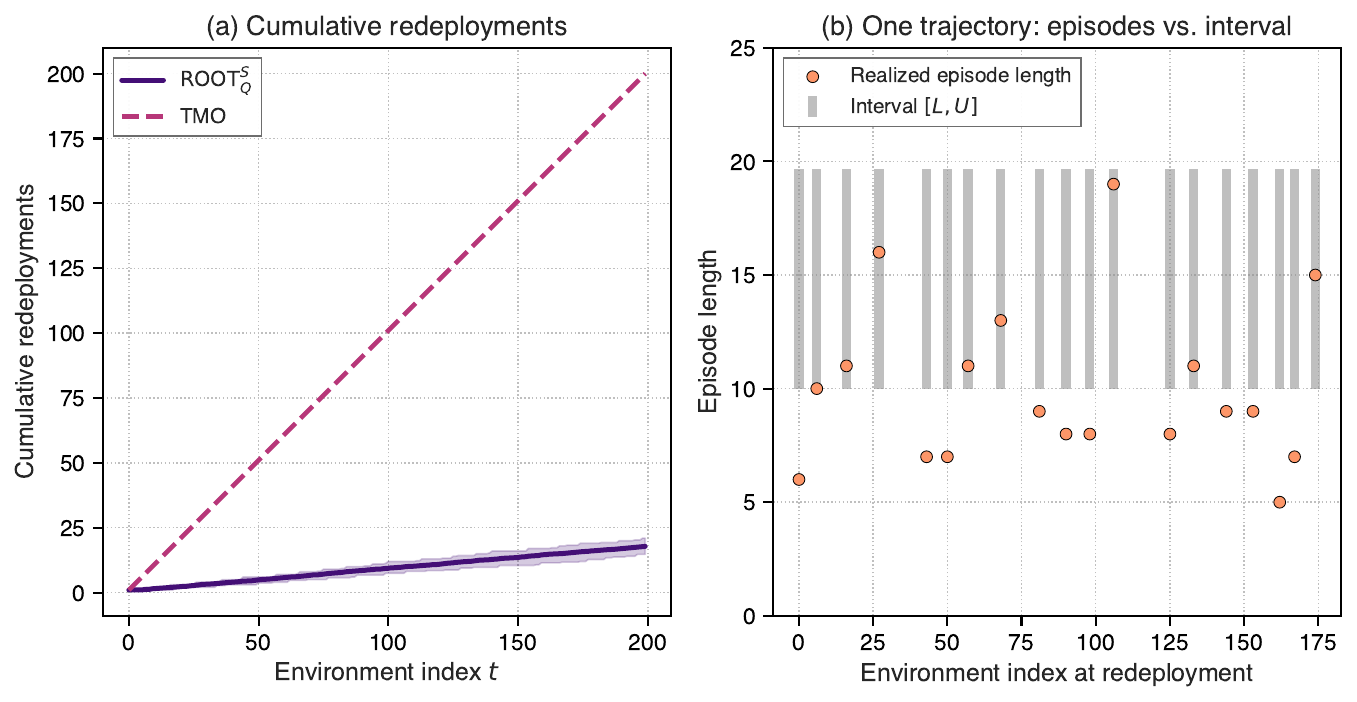}
\caption{A longitudinal illustration over 200 sequential environment
changes at \(d=10\), \(\sigma=0.10\), \(\delta=1\). (a) Cumulative number of
redeployments against environment index, comparing \(\mathrm{ROOT}^{S}_{Q}\)
(mean and 95\% band over 30 replicate trajectories) with TMO, which
redeploys at every change by construction. (b) One representative
trajectory's realized episode lengths under \(\mathrm{ROOT}^{S}_{Q}\)
(points) against the \((L,U)\) interval computed from the PSO
candidate of Section~\ref{sec:experimental-settings} at that redeployment
(grey bars).}
\label{fig:longitudinal}
\end{figure}

Figure~\ref{fig:longitudinal}(a) illustrates the cumulative deployment
behavior over time. Whereas TMO performs one redeployment after every
environmental change by construction, the proposed
\(\mathrm{ROOT}^{S}_{Q}\) policy intervenes only when the incumbent solution
becomes unacceptable. Averaged over 30 independent trajectories,
\(\mathrm{ROOT}^{S}_{Q}\) performs \(17.8\pm1.8\) redeployments during the
200-environment horizon, corresponding to an approximate 91\% reduction in
deployment events. The figure therefore makes explicit the operational
motivation underlying ROOT: environmental adaptation is achieved not by
continuous reoptimization, but by delaying expensive redeployments whenever
the current solution remains sufficiently effective.

Figure~\ref{fig:longitudinal}(b) follows one representative deployment
trajectory. After each redeployment, the analytical interval
\([L,U]\) is recomputed from the newly obtained deployment error and
compared with the realized duration of the subsequent deployment episode.
Across the 30 trajectories, 510 completed episodes were observed, of which
approximately 53\% had realized durations lying inside the corresponding
analytical interval.

This proportion should not be interpreted as an empirical coverage
probability. The analytical results bound the expected survival time rather
than individual realizations, and therefore no particular fraction of
episodes is expected to fall within the interval. Instead, the figure
illustrates how the analytical framework accompanies successive deployment
decisions, updating the expected operational lifetime after each
redeployment as the optimization error and the environmental history evolve.

Taken together, the longitudinal results complement the previous
experiments. The earlier sections established the correctness and practical
behavior of the analytical bounds for individual deployment decisions,
whereas the present experiment shows how those same bounds can be repeatedly
integrated into an ongoing deployment process. Rather than replacing the
optimizer, the proposed framework provides an analytical decision layer that
determines when continued deployment remains justified and when a new search
becomes operationally necessary.

\section{Concluding remarks}
\label{sec:concluding-remarks}

This work addresses a specific stage of robust optimization over time: the
period between deploying a solution and deciding whether it should be
replaced. We develop an analytical framework that relates the deployment
error of a fixed solution to its expected operational lifetime under
isotropic Gaussian environmental drift. Beyond the particular model studied
here, the analysis contributes to a broader understanding of deployment
persistence, a central concept underlying many ROOT problems and algorithms.

Our main contribution is a pair of rigorous bounds on expected survival
time. Together, they establish when a deployment horizon can be guaranteed,
when it is provably unattainable, and when the available analytical
information is insufficient to reach either conclusion. Beyond these
decision guarantees, the analysis reveals two fundamental properties of the
model: expected survival scales as \(\Theta(\sigma^{-2})\) in slowly
changing environments and collapses towards a single environmental change as
the dimension increases. The computational study confirms that these
behaviors remain visible across a broad range of parameter settings and
shows that the practical value of the lower and upper bounds depends
strongly on the operating regime.

More broadly, the results clarify the respective roles of optimization and
environmental dynamics. Improving the deployed solution can only reduce the
initial deployment error; it cannot compensate for unfavorable volatility,
dimension, or deployment requirements. In some regimes, the desired horizon
is unattainable even for the optimal deployment, implying that additional
search effort is fundamentally ineffective. Equally important, the
symmetric Gaussian model considered here does not generate a conflict
between present performance and future persistence: the current optimum is
also the survival-optimal fixed deployment. Any genuine trade-off between
TMO and ROOT therefore requires richer dynamics, predictive information, or
application-specific objectives.

The experiments also suggest that survival analysis can play a practical
decision-support role within ROOT. The bounds remain informative under
moderate departures from isotropy, identify situations in which parameter
uncertainty becomes operationally relevant, and provide interpretable
deployment assessments after optimization. In the longitudinal scenario,
this information translated into substantially fewer redeployments than a
policy that reoptimizes after every environmental change.

Several limitations remain. The analysis assumes a fixed deployed solution
and independent isotropic Gaussian environmental changes. Directional
drift, temporal dependence, heavy-tailed dynamics, non-spherical acceptance
regions, uncertain environmental parameters, and explicit switching costs
all fall outside the present theory. In addition, the analysis presumes that the optimizer is able to produce a candidate solution; we do not consider cases where the search fails entirely or where the estimation of the current optimum is subject to significant error. The decision framework is conditional on the supplied parameters, and our sensitivity analysis provides only a preliminary indication of how estimation errors affect the classification. Extending the framework to such
settings, together with replacing expectation-based guarantees by
risk-sensitive criteria or sequential decision policies, constitutes a
natural direction for future research.

Overall, we view this work as a first step towards a quantitative theory of
deployment persistence in ROOT. Survival time has long been used as a
descriptive measure of temporal robustness; the framework developed here
shows how it can also support deployment decisions before future
environments are observed. We hope that this perspective will help bridge
the gap between dynamic optimization and operational decision-making, and
serve as a foundation for more realistic deployment-aware ROOT methods.

\section*{Statements and Declarations}

\noindent\textbf{Funding.}
No funding information has been provided at this stage.

\noindent\textbf{Competing interests.}
The author declares no competing interests.

\noindent\textbf{Ethics approval and consent to participate.}
Not applicable.

\noindent\textbf{Consent for publication.}
Not applicable.

\noindent\textbf{Data availability.}
All generated numerical data needed to reproduce the tables and figures are
stored with the experimental code. A public archival URL will be added after acceptance.

\noindent\textbf{Code availability.}
The complete simulation and figure-generation code is available in the repository. A public archival URL and release
identifier will be added after acceptance.

\noindent\textbf{Author contributions.}
All authors equally contributed to the research and writing of the manuscript.

\bibliographystyle{unsrtnat}
\bibliography{references}

\end{document}